\documentclass[runningheads]{llncs}
\usepackage[T1]{fontenc}
\usepackage{graphicx}
\usepackage{algorithm}
\usepackage{algorithmic}
\usepackage[utf8]{inputenc} 
\usepackage[T1]{fontenc}    
\usepackage{hyperref}       
\usepackage{url}            
\usepackage{booktabs}       
\usepackage{amsfonts}       
\usepackage{nicefrac}       
\usepackage{microtype}      
\usepackage{lipsum}
\usepackage{fancyhdr}       
\usepackage{graphicx}       
\graphicspath{{media/}}     
\usepackage{booktabs} 
\usepackage{siunitx}  

\usepackage{amssymb}
\usepackage{subfig}  
\begin{document}

\newcommand{\inner}[2]{\langle #1, #2 \rangle}
\title{On Adversarial Training And The 1 Nearest Neighbor Classifier}
%
%
\author{Amir Hagai\inst{1}\orcidID{0009-0001-5453-6279} \and
Yair Weiss\inst{1}\orcidID{0000-0003-4643-5328}}

\institute{
Department of Computer Science, Hebrew University of Jerusalem \\
\email{\{amir.hagai,yair.weiss\}@mail.huji.ac.il}}
\maketitle              
\begin{abstract}
The ability to fool deep learning classifiers with tiny perturbations of the input has lead to the development of adversarial training in which the loss with respect to adversarial examples is minimized in addition to the training examples. While adversarial training improves the robustness of the learned classifiers, the procedure is computationally expensive, sensitive to hyperparameters and may still leave the classifier vulnerable to other types of small perturbations. In this paper we 
compare the performance of adversarial training to that of the simple 1 Nearest Neighbor (1NN) classifier. We prove that under reasonable assumptions, the 1NN classifier will be robust to {\em any} small image perturbation of the training images. In experiments with 135 different binary image classification problems taken from CIFAR10, MNIST and Fashion-MNIST we find that  1NN outperforms TRADES (a powerful adversarial training algorithm) in terms of average adversarial accuracy. In additional experiments with 69 robust models taken from the current adversarial robustness leaderboard, we find that 1NN outperforms almost all of them in terms of robustness to perturbations that are only slightly different from those used during training. Taken together, our results suggest that modern adversarial training methods still fall short of the robustness of the simple 1NN classifier. our code can be found at \url{https://github.com/amirhagai/On-Adversarial-Training-And-The-1-Nearest-Neighbor-Classifier}
\end{abstract}

\begin{figure}[htbp]
  \centering
  

  
  \begin{minipage}{0.99\columnwidth}
  \includegraphics[width=\linewidth]{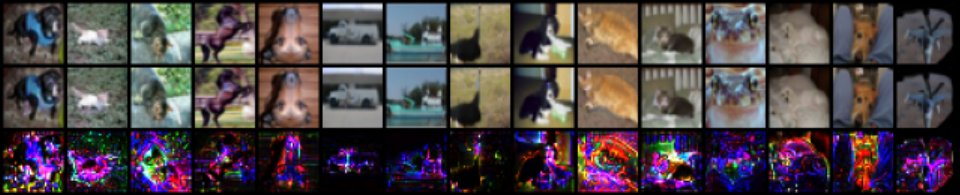}
    \end{minipage}

  \caption[]{Adversarial images for CIFAR10 training images for a CNN classifier obtained {\em with adversarial training}. First row is the original image, second row is the adversarial image and last row is the normalized difference. Even though the model was trained to be robust to any perturbation in which all pixel values are less than $8/255$, it can be easily fooled by imperceptible perturbations that are only slightly different from those that it saw during training. }
  \label{fig:three_figures}
\end{figure}

\section{Introduction}

Figure~\ref{fig:three_figures} illustrates the well-known phenomenon of adversarial examples\cite{szegedy2013intriguing}. The top row shows  CIFAR10 images that are correctly classified by a convolutional neural network, and the middle row shows nearly identical images that are misclassified by the same network. The bottom row shows the difference between the original images and the adversarial images, where the difference has been rescaled. 
This intriguing property \cite{szegedy2013intriguing} of modern neural networks has attracted significant interest in fields such as computer vision \cite{goodfellow2014explaining}, \cite{papernot2016transferability} \cite{carlini2017towards}, natural language processing \cite{zou2023universal}, and cybersecurity \cite{moradpoor2023threat}.

Algorithms for {\em adversarial training} (e.g. \cite{zhang2019theoretically,madry2017towards} ) try to address this lack of robustness of neural network classifiers by modifying the training procedure. At each iteration of training, the algorithms find adversarial examples of the current classifier and minimize a loss with respect to these adversarial examples along with the standard loss on the original training examples.

How well does adversarial training work? A commonly used metric is {\em adversarial accuracy} in which a classifier is said to correctly classify an image if it gives the correct output for all images in an $\epsilon$ ball around the image. Neural networks trained with standard training are highly vulnerable to small perturbations of their input so their adversarial accuracy is close to zero. For CIFAR10, adversarial training can increase the adversarial accuracy to above $60\%$ and with additional training data to above $70\%$ (e.g.~\cite{huang2022revisiting}, \cite{wang2023better}, \cite{peng2023robust}). 

Despite the impressive progress in adversarial accuracy when using adversarial training,  there are also several shortcomings. First, training using adversarial training is significantly more expensive, since an adversarial example needs to be found at each iteration of training. Second, the methods require several hyperparameters (e.g. the relative weights given to original examples vs. adversarial examples) and performance varies greatly with different choices. A third shortcoming is that adversarial training seems to mostly increase robustness to the specific types of perturbations that were used during training. 

Figure~\ref{fig:three_figures} illustrates this last shortcoming. The adversarial examples that are shown are for a {\em robust classifier}: the model was trained using adversarial training and we downloaded it from the RobustBench Adversarial Robustness leaderboard~\cite{croce2021robustbench}. When we test the model with the same perturbations that it was trained on, it is indeed quite robust.  However, when we slightly change the definition of allowed perturbations (see section 2 for the exact definition of allowed perturbations) then the adversarial accuracy is less than $5\%$. In other words for more than $95\%$ of the images, the classifier that was trained with adversarial training can be fooled with an imperceptible change of the input image. This illustrates the fact that "most adversarial training methods can only defend
against a specific attack that they are trained with. "\cite{nie2022diffusion}.

In this paper we seek to better understand the performance of adversarial training by comparing it to the simple 1 Nearest Neighbor (1NN) classifier.  We first formally define adversarial robustness and adversarial training and describe in detail one of the state-of-the-art algorithms for adversarial training: the TRADES~\cite{zhang2019theoretically} algorithm.  We then theoretically analyze the 1NN classifier and show that under reasonable assumptions, the 1NN classifier will be robust to {\em any} small image perturbation of the training images and will give high adversarial accuracy on test images as the number of training examples goes to infinity. We then perform experiments with 135 different binary image classification problems taken from CIFAR10, MNIST and Fashion-MNIST. We find that  1NN outperforms TRADES in terms of average adversarial accuracy with varied attackers. In additional experiments with 69  robust models taken from the RobustBench Adversarial Robustness Benchmark. 
We find that 1NN outperforms almost all of them in terms of
robustness to perturbations that are only slightly different from those
seen during training. Taken together, 
our results suggest that modern adversarial training methods still fall short of the robustness of the simple 1NN classifier.

\section{Adversarial Robustness and Adversarial Training}

\begin{figure}
\centerline{
\begin{tabular}{cc}
$\ell_2$ & $\ell_\infty$ \\
\includegraphics[width=0.5\columnwidth]{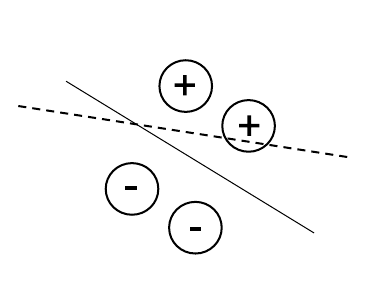} &
\includegraphics[width=0.5\columnwidth]{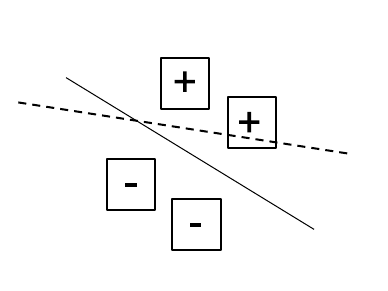}
\end{tabular}
}
\caption[]{The definition of robust accuracy depends on a norm $p$ and a radius $\epsilon$. A robust classifier needs to correctly classify all examples $x_i$ as well as all examples $x'$ that are in an $\epsilon$ ball around $x_i$. In the figure, the linear classifier denoted by the solid line has 100\% robust accuracy but the dashed line does not. In this paper we show that adversarial training produces classifiers that are robust to the specific values of $p,\epsilon$ that were used during training while 1NN is robust to any $p$ with sufficiently small $\epsilon$. }
\label{robustness-illustration}
\end{figure}

We start by formally defining the adversarial vulnerability of a classifier.   The clean accuracy of a classifier is simply the percentage of examples that are correctly classified:

\begin{definition}
    The clean accuracy of a classifier $f(x)$ on a set of examples $X=\{x_i,y_i\}_{i=1}^N$is:
    \begin{equation}
      CA(X;f)=\frac{1}{N} \sum_i [f(x_i) ==y_i]
    \end{equation}
\end{definition}

In contrast, the adversarial (or robust) accuracy is the percentage of examples that are correctly classified when an adversary is allowed to perturb an example by at most $\epsilon$.

\begin{definition}
    The robust  accuracy of a classifier $f(x)$ on a set of examples $X=\{x_i,y_i\}_{i=1}^N$ with a perturbation budget of $\epsilon$ using the  $\ell_p$ norm is:
    \begin{equation}
      RA(X,\delta,p;f)=\frac{1}{N} \sum_i \min_{v: \|v\|_p<\epsilon} [f(x_i+v) ==y_i]
    \end{equation}
\end{definition}

Note that the definition of robust accuracy depends on the type of allowed perturbation, i.e. it depends on an $\epsilon$ ball which is defined by a norm $p$ and an allowed radius $\epsilon$.   Figure~\ref{robustness-illustration} illustrates the definition. On the left the norm that is used is the $\ell_2$ norm and so robust accuracy requires that the decision boundary does not pass through circles around each point. The two linear classifiers that are shown both have 100\% clean accuracy (since they correctly classify all 4 points) but the classifier described by the dashed line has an adversarial accuracy that is only 75\% (since it passes thorough one of the circles). On the right the $\ell_\infty$ norm is used so that robust accuracy requires that the decision boundary does not pass through squares around each point. Intuitively, a good classifier should be robust for a wide range of possible values of $p$ and sufficiently small $\epsilon$. The standard practice is that $p$ and $\epsilon$ are "chosen such that the true
label should stay the same for each in-distribution input within the perturbation set "\cite{croce2021robustbench}. 

The intriguing property of neural networks~\cite{szegedy2013intriguing} is that {\em almost all state-of-the-art image classifiers (in terms of clean accuracy) have an adversarial accuracy that is close to zero}~\cite{mahmood2021robustnessvisiontransformersadversarial}. Referring again to figure~\ref{robustness-illustration}, this means that for almost all examples, the decision boundary of standard neural network classifiers passes through an $\epsilon$ ball around the example.

In an effort to solve this problem, adversarial training suggests to augment the original training examples with adversarial examples that are computed during training. A prime example of this approach is the TRadeoff-inspired Adversarial DEfense via Surrogate-loss minimization (TRADES) algorithm~\cite{zhang2019theoretically}\footnote{\url{https://github.com/yaodongyu/TRADES}}. This algorithm is summarized as Algorithm 1. At each iteration, for each training example $x_i$ the algorithm numerically searches for an adversarial example $x'$ in the $\epsilon$ ball around $x_i$. The parameters of the classifier $\theta$ are then updated in order to (1) output the true label on the original point $x_i$ and (2)  output the same label  for the point $x'$ in the $\epsilon$ ball and $x_i$. In other words,  the TRADES algorithm attempts to minimize the number of misclassified examples along with the number of times that the decision boundary passes through the $\epsilon$ ball around each training example. Note again that the algorithm requires as input the norm $p$ and the radius $\epsilon$. For the toy problem illustrated in figure~\ref{robustness-illustration}, if the algorithm receives as input $p=2$,  it will search for $x'$ in a circle around each training example and if it receives as input $p=\infty$  it will search for $x'$ in a square around each point.

How well does adversarial training work? In the standard evaluation used in the RobustBench Adversarial Robustness Benchmark~\url{https://robustbench.github.io/},  performance is measured separately on different values of $p$ and $\epsilon$. For example, one CIFAR10 leaderboard measures adversarial 
accuracy with attacks in an $\ell_{\infty}$ norm less than $8/255$ while another leaderboard measures adversarial accuracy with attacks with an $\ell_2$ norm less than $0.5$. The adversarial models are trained with a particular attacker in mind, and they are tested with the same attacker. Using this performance measure, the field has shown a steady improvement in performance over the last decade and the " the best entries of the $\ell_p$ leaderboards are still variants of PGD adversarial
training~\cite{zhang2019theoretically,madry2017towards} but with various enhancements (extra data, early stopping, weight averaging)." ~\cite{croce2021robustbench}. 

From an applications viewpoint, evaluating the robustness of algorithms using exactly the same time of perturbations that they saw during training seems overly optimistic. In the real world, systems may be faced with an almost infinite number of different perturbations: ideally we would like algorithms that perform well with one type of perturbation to also perform well with similar perturbations. As a case on point, consider the widely used $\ell_\infty, \epsilon=8/255$ perturbation: the adversary is allowed to change each pixel by at most $8/255$. These perturbations all have an $\ell_2$ norm that is less than $\sqrt{d} \frac{8}{255}$ where $d$ is the number of pixels. In other words, a classifier that was trained using perturbations less than $\frac{8}{255}$ using $\ell_\infty$ norm saw many examples of perturbations that are less than $\sqrt{d}\frac{8}{255}$ using an $\ell_2$ norm. Surely seeing all these examples would help it be robust to $\ell_2$ perturbations as well? 

As shown in figure~\ref{fig:three_figures}, this is not the case. Using TRADES to optimize  $\epsilon=\frac{8}{255}, \ell_\infty$ robustness yields a classifier that is easily fooled using $\epsilon=\sqrt{d}\frac{8}{255}, \ell_2$ perturbations. In other words, TRADES may yield a classifier that increases robustness to exactly the same type of perturbations that it saw during training, but not to perturbations that are only slightly different. In the next section, we consider a much simpler classifier that is guaranteed to be robust to a large range of perturbations.

\begin{algorithm}
  \caption{Adversarial training by TRADES}
  \label{alg:trades}
  \begin{algorithmic}[1]
    \REQUIRE Step sizes $\eta_1$ and $\eta_2$, batch size $m$, number of iterations $K$ in inner optimization, network architecture parametrized by $\theta$
    \ENSURE Robust network $f_\theta$
    \STATE Randomly initialize network $f_\theta$, or initialize network with pre-trained configuration
    \REPEAT
    \STATE Read mini-batch $B = \{x_1, \ldots, x_m\}$ from the training set
    \FOR{$i = 1$ to $m$ (in parallel)}
    \STATE $x'_i \leftarrow x_i + 0.001 \cdot \mathcal{N}(0, I)$, where $\mathcal{N}(0, I)$ is the Gaussian distribution with zero mean and identity variance
    \FOR{$k = 1$ to $K$}
    \STATE $x'_i \leftarrow \Pi_{B_{(x_i, \varepsilon)}}\left(\eta_{1}\cdot {\text{sign}}(\nabla_{x'_i} L(f_\theta(x_i), f_\theta(x'_i))) + x'_i\right)$ , where $\Pi$ is the projection operator
    \ENDFOR
    \ENDFOR
    \STATE $\theta \leftarrow \theta - \eta_2 \sum_{i=1}^{m} \nabla_\theta \left[L(f_\theta(x_i), y_i) + {L(f_\theta(x_i), f_\theta(x'_i))}/{\lambda}\right]/m$
    \UNTIL training converged
  \end{algorithmic}
\end{algorithm}

\begin{figure}[htbp]
  \centering
  
  \subfloat[TRADES\label{subfig:fig2}]{%
    \includegraphics[width=0.5\linewidth]{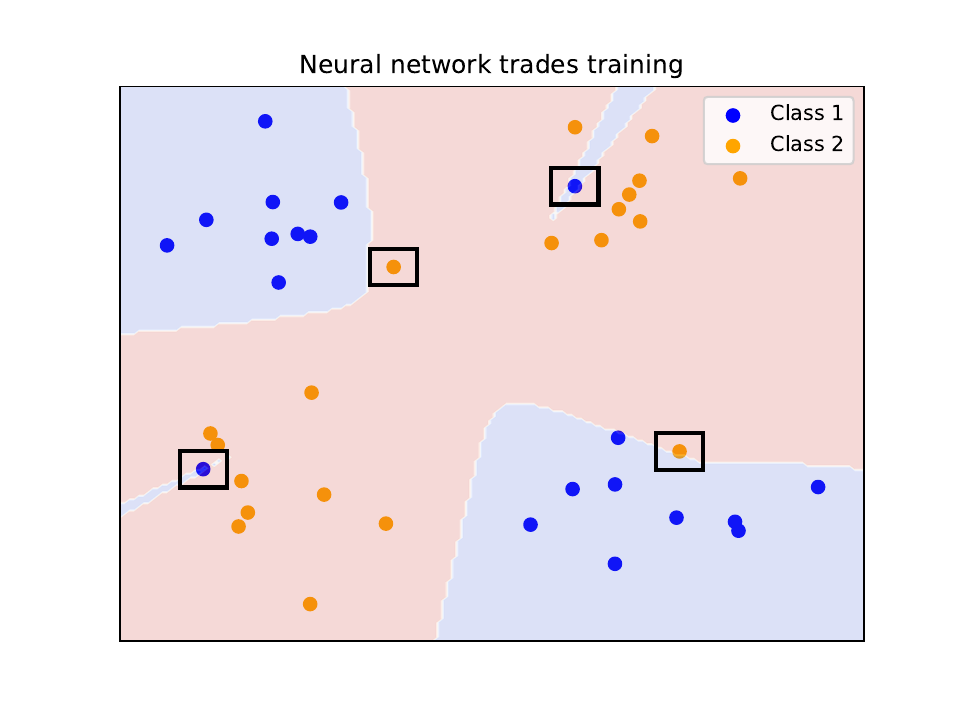}%
  }
  \subfloat[1NN\label{subfig:fig3}]{%
    \includegraphics[width=0.5\linewidth]{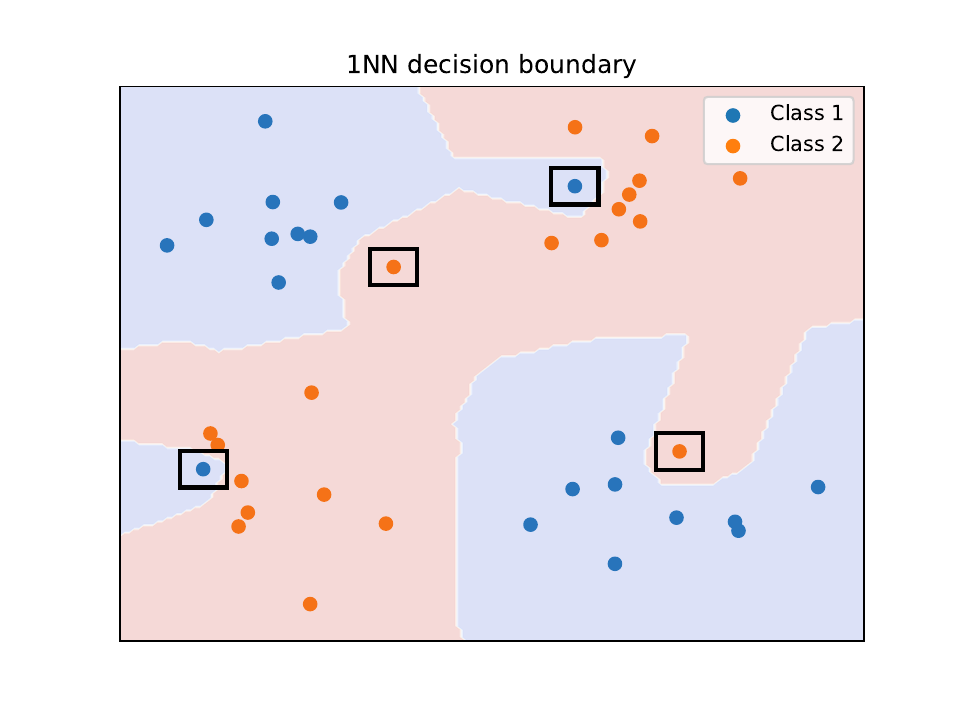}%
  }
  \caption{The decision boundary of TRADES (left) and 1NN on a simple 2D classification problem. The black square denotes the $\epsilon$ ball that defines adversarial accuracy (this was the $\epsilon$ ball used when the neural network was trained using TRADES). TRADES succeeds in making the decision boundary constant within the $\epsilon$ ball for some examples but not all. In contrast, the 1NN classifier achieves high robustness to {\em any} small perturbation of the training examples.} 
  \label{fig:overall_label}
\end{figure}

\section{The Robust Accuracy of 1NN}

\begin{figure}
\centerline{
\includegraphics[width=0.8\columnwidth]{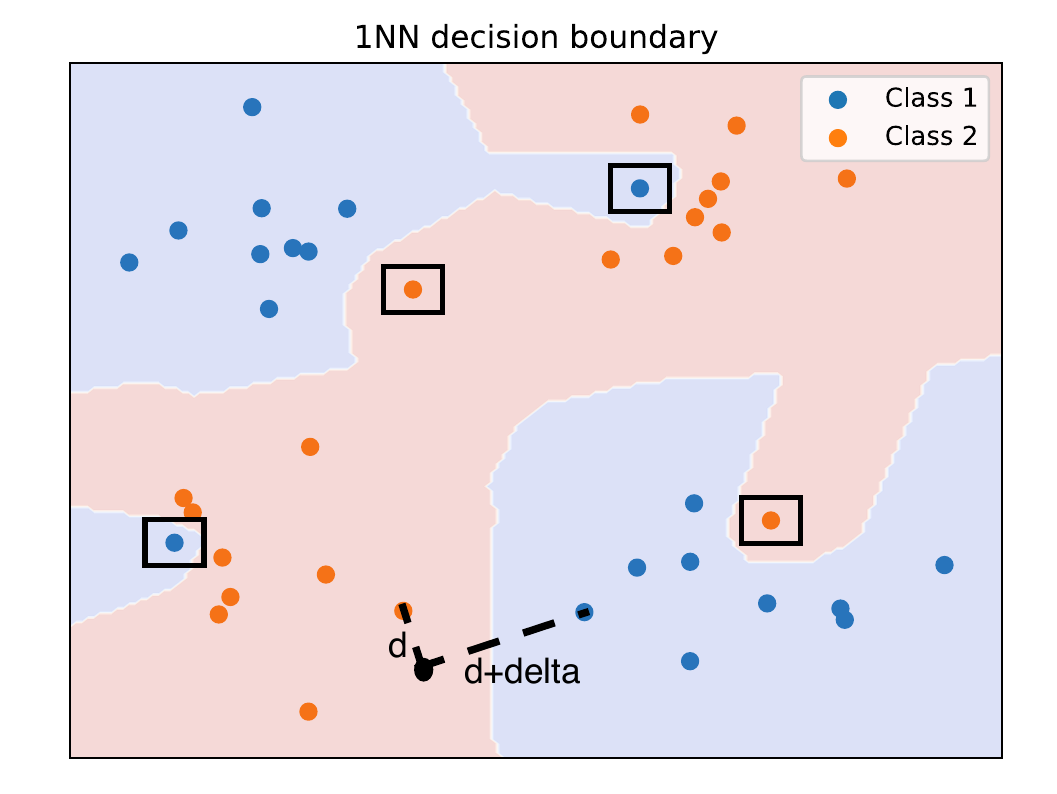}}
\caption[]{Illustration of the proof of theorem~\ref{thm:thm2}. If the distance to the closest point in the correct class is $d$ and the closest point in the incorrect class is $d+\delta$ then the 1NN classifier will be robust to any perturbation smaller than $\delta/2$.}
\label{proof-fig}
\end{figure}

The 1NN classifier has been extensively studied in machine learning (e.g.~\cite{duda2006pattern}) and is known to achieve a generalization error that is at most twice that of the Bayes optimal classifier as the number of examples goes to infinity. It has also been studied in the context of adversarial examples by~\cite{wang2019analyzingrobustnessnearestneighbors}. Our interest here is to use the 1NN as a baseline to which we can compare adversarial training. To motivate the comparison, consider the 2D classification problem shown in figure~\ref{fig:overall_label}. The left figure shows data for a binary classification problem  in two dimensions along with the decision boundary found by using TRADES to train a neural network classifier. As mentioned in the previous section, the goal of adversarial training is to maximize accuracy and to ensure that the decision boundary does not pass through the $\epsilon$ ball around the training examples. In figure~\ref{fig:overall_label} we plot this $\epsilon$ ball around a few training examples. It can be seen that even though TRADES attempts to push the decision boundary outside these $\epsilon$ balls it fails to do so for some of the points. For comparison, the right figure shows the decision boundary learned using the 1NN classifier. Note that it indeed succeeds in pushing the decision boundary outside the $\epsilon$ balls, even though there was no explicit "robust" training procedure. We now make this robustness more formal. 


\begin{definition} 

The 1NN classifier is a function that returns for each $x$, the class of $x_{NN}$ where $x_{NN}=\arg\min_{j} d(x, x_j)$ is the closest training example  to $x$ in terms of $\ell_2$ norm.

\label{def:1NN}
\end{definition}

We now show that the 1NN classifier will be provably robust to any small perturbations for test examples for which it is "confident", i.e. the nearest neighbor from the true class is significantly closer than the nearest neighbor from another class.

\begin{definition}
    $\delta$ 1NN Confident test examples. A test example is $\delta$ confidently classified by the 1NN classifier if $d(x, x_1) > d(x, x_0) + \delta$ where $x_0$ is the closest example in the correct class and $x_1$ is the closest example in another class.
\label{def:conf}
\end{definition}

For points for which the 1NN classifier is confident,   it is easy to show that it will be robust to small perturbations. The following theorem formalizes this intuition.

\begin{theorem}
Denote by $X_c$ a set of examples that are $\delta$ confidently classified by the 1NN classifier. robust accuracy of the 1NN classifier on this set  with $\ell_2$ norm, $\epsilon = \frac{\delta}{2}$ is $100\%$:
\[
 RA(X_c,\delta/2,2;1NN)=1
\]
\label{thm:thm2}
\end{theorem}

\begin{proof}
The proof is based on the triangle inequality and is illustrated graphically in figure~\ref{proof-fig}. A full proof is given in the appendix. 
\end{proof}

We now turn to characterizing conditions under which certain points will be confidently characterized by the 1NN classifier. We use the notion of the distance between the two classes: the minimal distance between a point in one class and a point in the other class. The larger that this distance is, the more confident the 1NN classifier will be in its prediction. This leads to the following theorem.
\begin{theorem}
If the distance between two classes is $d_0$ then the 1NN classifier achieves 100\% 
adversarial accuracy on the training set with $\ell_2$ norm and 
$\epsilon=\frac{d_0}{2}$ 
\label{thm:thm1}
\end{theorem}

This theorem follows directly from theorem~\ref{thm:thm2}. At first glance, guaranteeing robustness for {\em training} examples seems trivial, but note that for many modern classifiers, even the output on the training images can be changed with a tiny perturbation of the input. Specifically, all the examples in figure~\ref{fig:three_figures} are {\em training images} and yet a CNN that is trained using the state-of the art adversarial training algorithm can be fooled with tiny perturbations. The following theorem shows that the same property holds for {\em test} examples, but only when the number of training examples goes to infinity,

\begin{theorem}
If the distance between two classes is $d_0$ then the 1NN classifier achieves 100\% robust accuracy on the test set with $\ell_2$ norm  and $\epsilon=\frac{d_0}{2}$  as the number of training examples goes to infinity,
\label{thm:theorem3}
\end{theorem}
This theorem follows from the fact that as the number of training examples goes to infinity, the distance of a test example to the nearest neighbor in the same class goes to zero. This is the same property that is used to analyze the asymptotic performance of the 1NN classifier~\cite{duda2006pattern}.

The theory so far has focused on perturbations where the adversary is constrained to an $\epsilon$ ball in $\ell_2$. What happens with other $p$ norms? The easiest case is when $p \leq 2$.

\begin{corollary}
    for any $\text{p} \leq 2$, if the $\ell_2$ distance between the two classes is $d_0$ then the 1NN classifier achives  $100\%$ robust accuracy on  the training set with 
    $\ell_p$ norm and $\epsilon=\frac{d_0}{2}$.
    \label{coro1}
\end{corollary}

\begin{corollary}
    For any $\text{p} \leq 2$ the 1NN classifier achieves  $100\%$ robust accuracy with $\ell_p$ norm and $\epsilon=\frac{\delta}{2}$ on  test examples   that are classified $\delta$ confidently .
    \label{coro2}
\end{corollary}

Both of these follow from the fact that the unit ball with $\ell_p$ norm is contained
within the unit ball with $\ell_2$ norm for any $p \leq 2$

Finally, we can summarize all of our results for all $p$ norms in the following corollary:
\begin{corollary}
    Suppose that the $\ell_2$ distance between the two classes is $2\delta\sqrt{N}$ where N is the dimensionality. For any p the adversarial accuracy with $\ell_p$ norm and $\epsilon=\delta$ is $100\%$ on the training set and as the number of examples goes to infinity the adversarial accuracy on the test set will also be $100\%$
    \label{coro3}
\end{corollary}

\begin{proof}
This follows from the fact that for any p the $\ell_p$ unit ball is contained within
the $\ell_{\infty}$ unit ball and the $\delta$ ball for $\ell_{\infty}$ is contained within the $\sqrt{N}\delta$ ball for $\ell_2$
\end{proof}

The corollary can be illustrated by returning again to figure~\ref{fig:overall_label}. The 1NN classification boundary does not pass though the squares shown around the training points indicating it is robust to small $\ell_\infty$ perturbations but it also does not pass through small circles or diamonds around each point, indicating it is also robust to small $\ell_2$ or $\ell_1$ perturbations. In contrast, adversarial training (shown on the left) is trained with a particular perturbation and fails to be 100\% robust even for that perturbation.

As mentioned above, the adversarial accuracy of 1NN was recently analyzed in~\cite{wang2019analyzingrobustnessnearestneighbors}. They showed that when the distributions of the two classes overlap, then the 1NN classifier will have adversarial accuracy that approaches zero, while a KNN classifier with large K can achieve the same adversarial accuracy of the Bayes-Optimal classifier. Our theoretical results, on the other hand, assume that the two distributions are separated. We believe that the assumption of separated distributions is relevant for the small values of $\epsilon$ that are commonly used to test for adversarial vulnerability.  As mentioned above, in these settings $\epsilon$ is chosen so small so that {\em any} point in the $\epsilon$ ball around a point $x$ from the training distribution has the same label as $x$.  This is equivalent to the assumption that the distributions are separated. 

Perhaps a more significant difference between our work and that of ~\cite{wang2019analyzingrobustnessnearestneighbors} is that we are primarily interested in the 1NN classifier as a baseline with which we can compare modern adversarial training algorithms. In the next section, we perform extensive empirical comparisons. 

\section{Experiments}

\begin{table}[ht]
\centering

\begin{tabular}{@{}lSS@{}}
\toprule
Algorithm & {Training (seconds)} & {Inference  (seconds)} \\ 
\midrule
TRADES & 7467.693 & 0.00009 \\
early stop & 99.84268 & 0.00009 \\
1NN faiss & 0.34034  & 0.00001 \\

\bottomrule\\
\end{tabular}
\caption{Running Times for Training and Inference, on GeForce RTX 2080 Ti, with ResNet18 over a binary CIFAR10 problem. For 1NN we use the FAISS package which builds a datastructure to enable fast nearest neighbor search, and we list the time required to build the datastructure as training time. TRADES is about 75 times slower than standard neural network training and about 20,000 times slower than 1NN.}
\label{tab:running_times}

\end{table}

\begin{table}
\centerline{
\begin{tabular}{lcccccc}
\toprule
 & 1-NN & Normal Training & Early Stop  & TRADES  & TRADES  & TRADES \\
  & & & & Best & Recom & Both \\
\midrule
Clean Acc & 75.11 $\pm$ 7.05 & 86.02 $\pm$ 8.95 & 74.29 $\pm$ 12.69 & 92.71 $\pm$ 7.84 & \textbf{95.24 $\pm$ 4.98} & 94.36 $\pm$ 4.65 \\
$\ell_2$ Rob Acc (Train) & \textbf{100.0 $\pm$ 0.0} & 13.47 $\pm$ 13.91 & 36.29 $\pm$ 14.57 & 29.31 $\pm$ 14.79 & 14.36 $\pm$ 11.56 & 29.64 $\pm$ 14.27 \\
$\ell_2$ Rob Acc (Test) & \textbf{52.96 $\pm$ 9.34} & 13.49 $\pm$ 14.07 & 36.02 $\pm$ 14.55 & 27.87 $\pm$ 13.83 & 14.07 $\pm$ 11.55 & 28.62 $\pm$ 13.91 \\
$\ell_\infty$ Rob Acc (Train) & \textbf{100.0 $\pm$ 0.0} & 32.47 $\pm$ 18.35 & 53.49 $\pm$ 11.09 & 74.64 $\pm$ 16.33 & 61.60 $\pm$ 21.26 & 77.20 $\pm$ 19.40 \\
$\ell_\infty$ Rob Acc (Test) & 58.76 $\pm$ 9.61 & 32.49 $\pm$ 18.35 & 53.18 $\pm$ 10.99 & \textbf{73.36 $\pm$ 17.06} & 58.51 $\pm$ 22.05 & 72.58 $\pm$ 19.09 \\
$\ell_\infty + \ell_2$ RA (Train) & \textbf{100.0 $\pm$ 0.0} & 22.98 $\pm$ 15.57 & 44.87 $\pm$ 11.97 & 51.98 $\pm$ 14.06 & 38.00 $\pm$ 14.69 & 53.44 $\pm$ 15.60 \\
$\ell_\infty + \ell_2$ RA (Test) & \textbf{55.84 $\pm$ 9.56} & 22.98 $\pm$ 15.66 & 44.60 $\pm$ 11.82 & {50.60 $\pm$ 13.69} & 36.29 $\pm$ 14.61 & 50.67 $\pm$ 15.53 \\
\bottomrule
\\
\end{tabular}}
\caption[]{Results on 45 binary classification problems taken from CIFAR10 (e.g. "dog" vs. "frog").  For each binary problem we compute the accuracy using all test points, and each entry in the table shows the average and standard deviation of the accuracy over 45 problems. On the training set, 1NN achieves 100\% robust accuracy while none of the adversarial training methods do so. On the test set, TRADES outperforms 1NN in terms of the perturbations on which it was trained ($\ell_\infty,\epsilon=\frac{8}{255}$) but not on slightly different perturbations ($\ell_2,\epsilon=1.74$). In terms of robustness to both perturbations (bottom row), the simple 1NN outperforms adversarial training.}\
\label{tab:cifar10}
\end{table}

\begin{table}
\centerline{
\begin{tabular}{lcccccc}
\toprule
& 1-NN & Normal Training & Early Stop  & TRADES  & TRADES  & TRADES \\
  & & & & Best & Recom & Both \\
\midrule
Clean Acc  & 99.56 $\pm$ 0.57 & 99.87 $\pm$ 0.40 & \textbf{100.00 $\pm$ 0.00} & 99.87 $\pm$ 0.40 & 99.80 $\pm$ 0.50 & 99 $\pm$ 2.1\\
$\ell_2$ RA (Train) & \textbf{99.43 $\pm$ 2.39} & 0.20 $\pm$ 1.20 & 0.03 $\pm$ 0.19 & 0.60 $\pm$ 2.65 & 0.00 $\pm$ 0.00 & 2.82 $\pm$ 6.83 \\
$\ell_2$ RA (Test) & \textbf{85.42 $\pm$ 16.19} & 0.07 $\pm$ 0.25 & 0.00 $\pm$ 0.00 & 0.71 $\pm$ 2.99 & 0.00 $\pm$ 0.00 & 3.57 $\pm$ 8.39 \\
$\ell_\infty$ RA (Train) & \textbf{100.00 $\pm$ 0.00} & 4.78 $\pm$ 8.60 & 11.59 $\pm$ 12.86 & 97.67 $\pm$ 2.38 & 92.13 $\pm$ 4.65 & 4.93 $\pm$ 8.72 \\
$\ell_\infty$ RA (Test) & 79.40 $\pm$ 17.32 & 3.29 $\pm$ 7.42 & 7.87 $\pm$ 11.05 & \textbf{97.09 $\pm$ 2.26} & 89.04 $\pm$ 7.65 & 5.75 $\pm$ 9.46 \\
$\ell_\infty + \ell_2$ RA (Train) & \textbf{99.72 $\pm$ 1.20} & 2.49 $\pm$ 4.44 & 6.06 $\pm$ 6.56 & 49.13 $\pm$ 1.91 & 46.07 $\pm$ 2.32 & 3.88 $\pm$ 7.34 \\
$\ell_\infty + \ell_2$ RA (Test) & \textbf{82.41 $\pm$ 16.00} & 1.68 $\pm$ 3.71 & 4.07 $\pm$ 5.57 & 48.90 $\pm$ 1.97 & 44.52 $\pm$ 3.83 & 4.66 $\pm$ 8.36 \\
\bottomrule \\
\end{tabular}}
\caption[]{Results on 45 binary classification problems taken from MNIST (e.g. "8" vs. "7"). For each binary problem we compute the accuracy using 100 test points, and each entry in the table shows the average and standard deviation of the accuracy over 45 problems.  On the training set, 1NN achieves 100\% robust accuracy while none of the adversarial training methods do so. On the test set, TRADES outperforms 1NN in terms of the perturbations on which it was trained ($\ell_\infty,\epsilon=0.3$) but not on different perturbations ($\ell_2,\epsilon=8.4$). In terms of robustness to both perturbations (bottom row), the simple 1NN outperforms adversarial training. }
\label{tab:mnist}
\end{table}

\begin{table}
\centerline{
\begin{tabular}{lcccccc}
\toprule
 & 1-NN & Normal Training & Early Stop  & TRADES  & TRADES  & TRADES \\
  & & & & Best & Recom & Both \\
\midrule
Clean Acc & 97.42 $\pm$ 4.37 & 98.84 $\pm$ 2.55 & 99.81 $\pm$ 0.49 & 97.02 $\pm$ 5.77 & 96.91 $\pm$ 4.32 & \textbf{99.89 $\pm$ 0.32} \\
$\ell_2$ Rob Acc (Train) & \textbf{92.77 $\pm$ 12.91} & 2.80 $\pm$ 5.77 & 0.00 $\pm$ 0.00 & 2.16 $\pm$ 6.44 & 0.05 $\pm$ 0.21 & 0.33 $\pm$ 1.54 \\
$\ell_2$ Rob Acc (Test) & \textbf{71.16 $\pm$ 29.72} & 4.00 $\pm$ 7.98 & 0.08 $\pm$ 0.27 & 2.07 $\pm$ 6.06 & 0.09 $\pm$ 0.47 & 0.13 $\pm$ 0.55 \\
$\ell_\infty$ Rob Acc (Train) & \textbf{96.01 $\pm$ 7.13} & 4.82 $\pm$ 8.49 & 12.53 $\pm$ 10.74 & 76.76 $\pm$ 24.60 & 63.80 $\pm$ 28.44 & 4.53 $\pm$ 8.53 \\
$\ell_\infty$ Rob Acc (Test) & 72.68 $\pm$ 26.01 & 5.56 $\pm$ 9.42 & 12.50 $\pm$ 10.82 & \textbf{79.13 $\pm$ 25.59} & 64.59 $\pm$ 29.23 & 3.44 $\pm$ 7.39 \\
$\ell_\infty + \ell_2$ RA (Train) & \textbf{94.39 $\pm$ 9.78} & 3.81 $\pm$ 6.74 & 5.76 $\pm$ 5.47 & 39.46 $\pm$ 13.17 & 31.92 $\pm$ 14.23 & 2.43 $\pm$ 4.49 \\
$\ell_\infty + \ell_2$ RA (Test) & \textbf{71.92 $\pm$ 27.68} & 4.78 $\pm$ 8.19 & 6.29 $\pm$ 5.43 & 40.60 $\pm$ 13.57 & 32.34 $\pm$ 14.65 & 1.79 $\pm$ 3.71 \\
\bottomrule\\
\end{tabular}}
\caption[]{Results on 45 binary classification problems taken from Fashion-MNIST (e.g. "dress" vs. "coat"). For each binary problem we compute the accuracy using 100 test points, and each entry in the table shows the average and standard deviation of the accuracy over 45 problems. On the training set, 1NN achieves over 90\% robust accuracy while none of the adversarial training methods do so. On the test set, TRADES outperforms 1NN in terms of the perturbations on which it was trained ($\ell_\infty,\epsilon=0.3$) but not on slightly different perturbations ($\ell_2,\epsilon=8.4$). In terms of robustness to both perturbations (bottom row), the simple 1NN outperforms adversarial training.}
\label{tab:fashion}
\end{table}

To evaluate the performance of 1NN and adversarial training on real datasets, we carried out an extensive series of experiments. In our first set of experiments, we used three standard datasets: CIFAR-10, MNIST and Fashion-MNIST. Rather than simply running the algorithms on the three datasets, we used these datasets to create 135 binary classification problems. For example, since CIFAR10 includes 10 possible categories, we can create 45 different binary classification problems that are defined by choosing two labels (e.g. "dog" vs "frog", "car" vs. "truck"). Similarly, we can create 45 different binary problems from MNIST and Fashion-MNIST giving us 135 problems in total. Our motivation in using binary classification problems was to compare the performance of the different algorithms on a large number of learning problems, rather than just three.

 As mentioned previously, defining robust accuracy requires defining a norm and a radius that define the $\epsilon$ ball around each point.  We evaluated all algorithms using both an $\ell_\infty$ norm and an $\ell_2$ norm.  In the case of CIFAR10, for the $\ell_\infty$ norm we used the commonly used radius of $\frac{8}{255}$: the adversary is allowed to change each pixel by at most $\frac{8}{255}$ while for the MNIST and Fashion-MNIST problems we allowed the adversary to change each pixel by at most $0.3$ (this is the default used for MNIST in the code and the paper of TRADES~\cite{zhang2019theoretically}).  For the $\ell_2$ norm, we chose $\epsilon$ as the norm of a perturbation in which all pixels have been perturbed by the maximal amount allowed to the $\ell_\infty$ attacker. This gave a radius of $1.74$ for the CIFAR10 problems and $8.4$ for the MNIST problems. For CIFAR10, we verified visually that these values of $\epsilon$ were sufficiently small that the category does not change within an $\epsilon$ ball around a training example. Figure~\ref{fig:three_figures} shows some examples. It can be seen that the images with an allowed $\ell_2,\epsilon=1.74$ perturbation (middle row) are almost indistinguishable from the original images (top row). 
 
For each of the 135 problems, we evaluated the following algorithms:
\begin{enumerate}
\item The 1NN classifier. Note that the 1NN classifier always uses the $\ell_2$ norm to compute distances.
\item TRADES with recommended hyperparameters based on a ResNet18 architecture (as in \cite{pang2020bag}). 
\item TRADES with the best problem-specific hyperparameters. For each binary problem, We tried a number of possible combinations of the hyperparameters and chose the one that gave highest $\ell_\infty$ adversarial accuracy. For the CIFAR10 problems we tried 600 different hyperparameter settings while for the MNIST and Fashion-MNIST we tried a smaller number of possible settings.
\item Standard training of the same ResNet18 architecture. 
\item Standard training with early stopping (20 epochs). 
\item A version of TRADES which attempts to be robust to both $\ell_\infty$ and $\ell_2$ attacks. In this version, at each iteration we find two  adversarial examples $x'$: one is constrained to be within the $\ell_2$ ball around the original example and the second is constrained to be within the $\ell_\infty$  ball.  We refer to this algorithm as "TRADES-Both".
\end{enumerate}

For the 1NN classifier we can calculate the robust accuracy exactly but 
to evaluate the CNN's robust accuracy we need to perform optimization.  We primarily used the 
APGD-CE\cite{croce2020reliable} attack which is the recommended attack in~\cite{croce2021robustbench}. In initial experiments, we also tried the attacks used in CleverHans\cite{papernot2018cleverhans} and Foolbox\cite{rauber2017foolboxnative}. However, our initial experiments suggested that APGD-CE was either equally effective or superior in terms of finding adversarial examples, and the evaluation was faster.

Table~\ref{tab:running_times} shows the training and inference times for each relevant method. For the 1NN we used the FAISS package~\cite{johnson2019billion} which builds an index over the training example in order to facilitate fast nearest neighbor search and we use that as the equivalent of "training time". Clearly 1NN is orders of magnitude faster.

The results for the 45 problems from CIFAR10 are shown in table~\ref{tab:cifar10}. For each binary problem we compute the accuracy using
all test points, and each entry in the table shows the average and standard
deviation of the accuracy over 45 problems.
The results are consistent with our theory and show that:
\begin{itemize}
\item For all 45 problems, 1NN achieved 100\% robust accuracy on the train set for both $\ell_\infty$ and $\ell_2$ norms. In contrast, with the recommended hyperparameters for TRADES, the robust accuracy on the train was always far from 100\%: with the $\ell_2$ norm the robust accuracy on the training set was below 50\% for all 45 problems. Recall that the robust accuracy on the training set is guaranteed to be 100\% for 1NN based on our theory (provided the classes are well separated). Recall also that TRADES explicitly tries to maximize the robust accuracy on the training set with $\ell_\infty$ norm, and even when we measure robust accuracy using the same $\epsilon$ balls that were used during training, the average robust accuracy (over the 45 problems) is around 61 \%, far below that of the simple 1NN.
\item The performance of TRADES is strongly dependent on hyperparameters. When we use the same hyperparameters for all 45 problems (set to their recommended setting for CIFAR10), the robust accuracy with $\ell_\infty$ norm on the test set is not any better than it is for the simple 1NN classifier (around 58\% for both algorithms). Thus even though TRADES is orders of magnitude slower than 1NN, there is no significant improvement in robust accuracy even when we measure it using the same perturbations that were used in training. 
\item When we allow TRADES to choose different hyperparameters for each of the 45 problems, the algorithm does indeed significantly improve robust accuracy on the test set using the $\ell_\infty$ norm and it outperforms the simple 1NN (73\% on average over the 45 problems, as opposed to 58\% for 1NN). 
\item Even when we allow TRADES to choose different hyperparameters for each of the 45 problems, it does not learn a classifier that is robust to perturbations that are only slightly different from those that it saw during training. The robust accuracy on the test set with $\ell_2$ norm is around $27\% \pm 13$ indicating that for almost all 45 problems the robust accuracy with $\ell_2$ norm for TRADES was below 50\%.   In contrast, the simple 1NN classifier is guaranteed by our theorem to be robust to any small perturbation for test examples that are confidently classified, and indeed the robust accuracy with $\ell_2$ norm is much better than that of TRADES when averaged over all 45 problems.
\end{itemize}

The results for MNIST and Fashion-MNIST are given in tables~\ref{tab:mnist},\ref{tab:fashion} and show a similar trend. 
Again, we find that TRADES is highly sensitive to hyperparameters, fails to be 100\% robust on the training data, and is inferior to the simple 1NN on test data when two types of perturbations are considered.
Overall, we believe that that these experiments clearly show an advantage of simple 1NN over adversarial training in terms of robustness to varied perturbations. Furthermore, they also highlight what we see as the main advantage of 1NN over adversarial training: while adversarial training increases robustness to the particular perturbations on which it was trained, 1NN is robust to any small perturbation under reasonable assumptions.

In a second set of experiments we 
compared the performance of 69 state-of-the-art pretrained models (all of the available CIFAR10, $\ell_\infty$ models at the RobustBench site~\cite{croce2021robustbench})  to that of the simple 1NN classifier. These pretrained models include a range of architectures and algorithms that are built on the basis of adversarial training but also include methods that take advantage of additional training data (e.g. one of the current leaders uses 50 million images generated by a diffusion model in addition to the CIFAR10 images~\cite{pmlr-v202-wang23ad}).

As seen in figure~\ref{fig:overall_label} the behaviour is similar to what we saw in the binary experiments: the simple 1NN is better than almost all the pretrained models in terms of $\ell_2$ robust accuracy and in terms of $\ell_{\infty}$ accuracy on the training set. Thus even when the models are trained with millions of additional images and adversarially generated examples, they still fail to be robust for small perturbations that are only slightly different from those seen during training. In contrast, the 1NN classifier is robust to {\em any} small perturbation as predicted by our theory. 

As a visual comparison of the 1NN approach and TRADES, figure~\ref{fig:1NN-examples} shows the adversarial examples needed to fool the 1NN classifier for the same training examples shown in figure~\ref{fig:three_figures}. Whereas TRADES is easily fooled with small perturbations that are only slightly different from those  that it saw during training, the 1NN model cannot be fooled with small perturbations. Rather, fooling 1NN requires much larger changes and almost all pixels should be changed by over $\frac{8}{255}$.

\begin{figure}[htbp]
  \centering
  
  \begin{minipage}{0.4\columnwidth}

    \includegraphics[width=\linewidth]{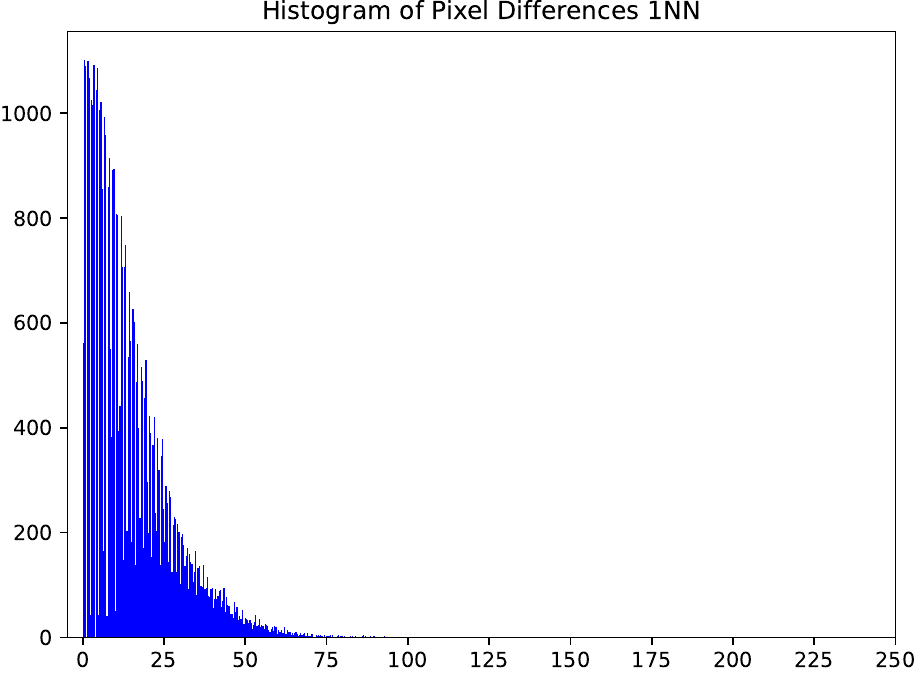}
  \end{minipage}
  
  \begin{minipage}{0.8\columnwidth}
  \includegraphics[width=\linewidth]{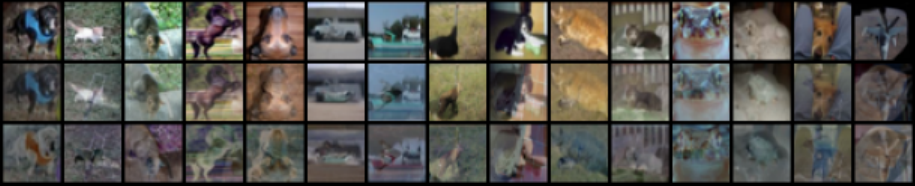}
    \end{minipage}

  \caption[]{Adversarial images for 1NN for the same CIFAR10 training images shown in~\protect{\ref{fig:three_figures}}.  First row is the original image, second row is the adversary image and last line is the normalized difference. As can be seen in the histogram of pixel differences on top, to fool the 1NN nearest neighbor classifier, almost all pixels should be changed by over $8/255$.} 
  \label{fig:1NN-examples}
\end{figure}

\begin{figure}
\centerline{
\begin{tabular}{cccc}
\includegraphics[width=0.25\linewidth]{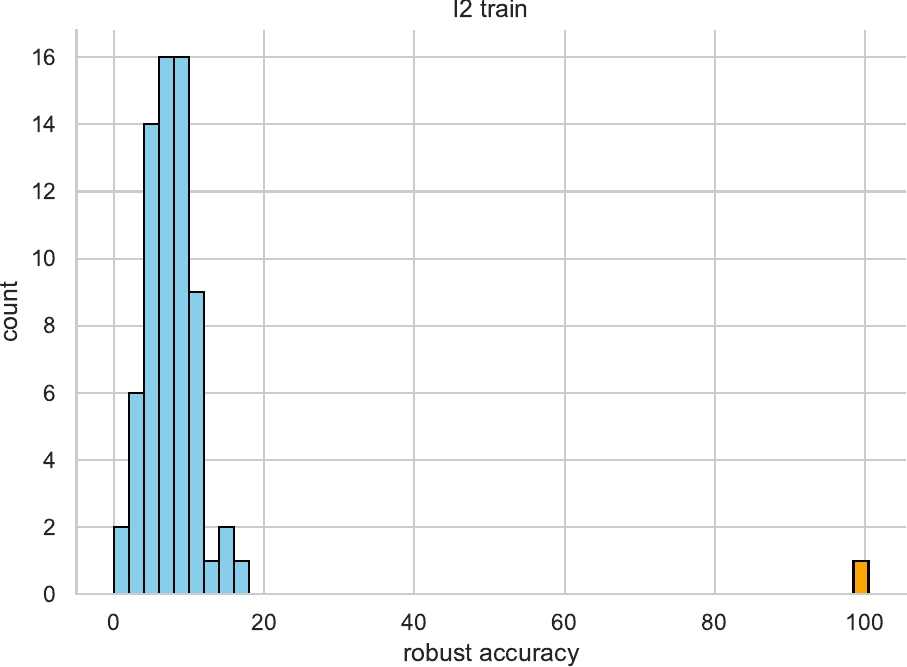} & 
\includegraphics[width=0.25\linewidth]{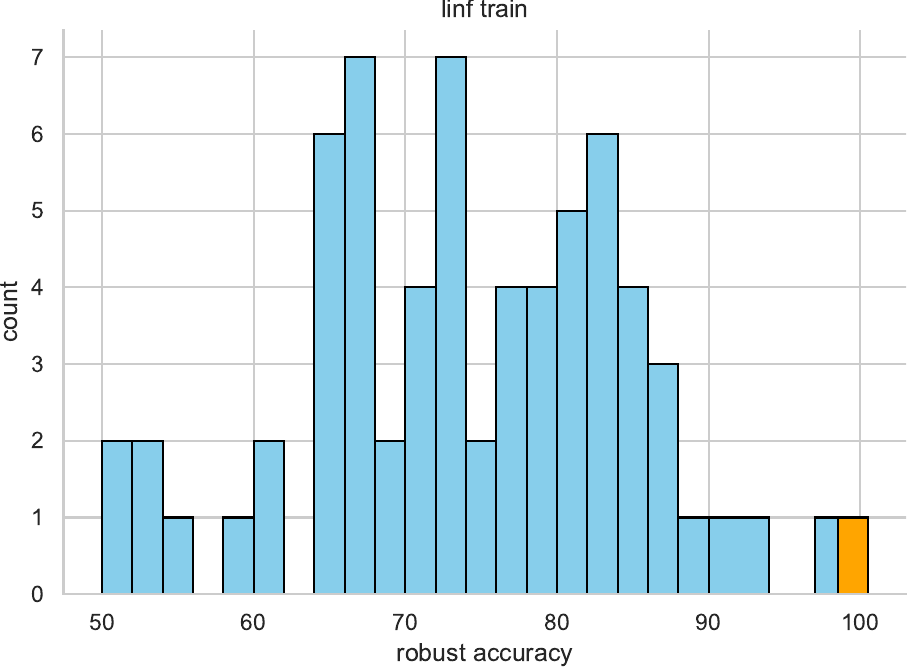} & 
 \includegraphics[width=0.25\linewidth]{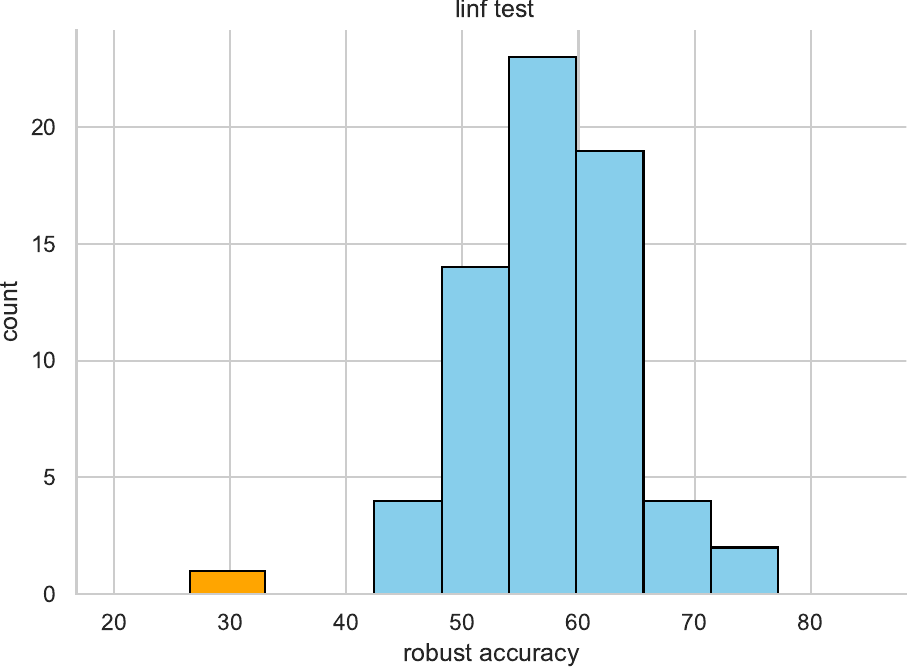} &
  \includegraphics[width=0.25\linewidth]{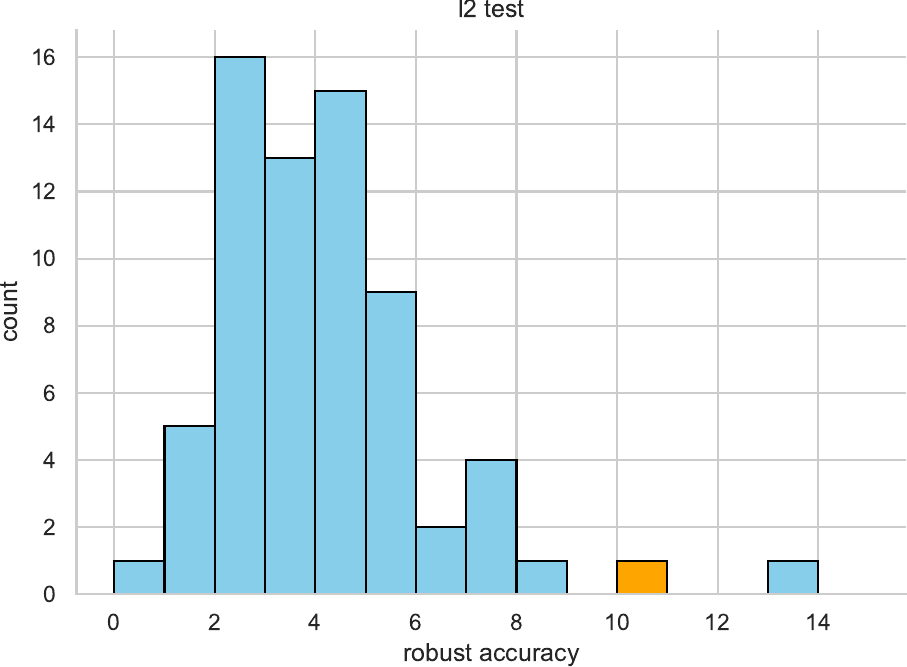} \\
  $\ell_2$ RA train & 
   $\ell_\infty$ RA train &
  $\ell_\infty$ RA test& $\ell_2$ RA test
\end{tabular}
}
\caption[]{Comparison of 1NN (yellow) and 69 pretrained models for CIFAR10 downloaded from the current $\ell_{\infty}$ robustness \href{https://robustbench.github}{leaderboard} (blue). All models are trained using adversarial training using $\ell_\infty$ attacks and indeed they increase the adversarial accuracy using this measure significantly. However, when these models are evaluated using $\ell_2$ attacks, they are inferior to the 1NN classifier. }
  \label{fig:pretrained}
\end{figure}


\section{Discussion}
Despite the large amount of attention that the research community has invested in adversarial examples, the problem is still unsolved and even state of the art methods based on adversarial training can be easily fooled with tiny perturbations of the input. In this paper we have compared adversarial training to the simple 1NN classifier. We have analyzed the 1NN classifier and shown proven that under reasonable assumptions, the 1NN will have 100\% adversarial accuracy on the training set and as the number of training examples goes to infinity, it will also attain 100\% adversarial accuracy on the test set for any $p$ norm. In extensive experiments on 135 different image classification problems, the simple 1NN classifier outperforms adversarial training in terms of robust accuracy when both $\ell_2$ and $\ell_\infty$ norms are used. 

We see the simple 1NN classifier as an example of a simple algorithm that provides robustness to any small perturbation and therefore serves as an alternative to adversarial training. We do not, however, advocate using the 1NN classifier in pixel space for image classification since obviously the performance for finite training data is far from state-of-the-art. We believe that there are more powerful algorithms that are based on 1NN in other feature spaces that can improve accuracy and still guarantee robustness and this is a promising direction for future research.


\bibliographystyle{unsrt}  
\bibliography{references}

\begin{thebibliography}{10}

\bibitem{szegedy2013intriguing}
Christian Szegedy, Wojciech Zaremba, Ilya Sutskever, Joan Bruna, Dumitru Erhan, Ian Goodfellow, and Rob Fergus.
\newblock Intriguing properties of neural networks.
\newblock {\em arXiv preprint arXiv:1312.6199}, 2013.

\bibitem{goodfellow2014explaining}
Ian~J Goodfellow, Jonathon Shlens, and Christian Szegedy.
\newblock Explaining and harnessing adversarial examples.
\newblock {\em arXiv preprint arXiv:1412.6572}, 2014.

\bibitem{papernot2016transferability}
Nicolas Papernot, Patrick McDaniel, and Ian Goodfellow.
\newblock Transferability in machine learning: from phenomena to black-box attacks using adversarial samples.
\newblock {\em arXiv preprint arXiv:1605.07277}, 2016.

\bibitem{carlini2017towards}
Nicholas Carlini and David Wagner.
\newblock Towards evaluating the robustness of neural networks.
\newblock In {\em 2017 ieee symposium on security and privacy (sp)}, pages 39--57. Ieee, 2017.

\bibitem{zou2023universal}
Andy Zou, Zifan Wang, J~Zico Kolter, and Matt Fredrikson.
\newblock Universal and transferable adversarial attacks on aligned language models.
\newblock {\em arXiv preprint arXiv:2307.15043}, 2023.

\bibitem{moradpoor2023threat}
Naghmeh Moradpoor, Leandros Maglaras, Ezra Abah, and Andres Robles-Durazno.
\newblock The threat of adversarial attacks against machine learning-based anomaly detection approach in a clean water treatment system.
\newblock In {\em 2023 19th International Conference on Distributed Computing in Smart Systems and the Internet of Things (DCOSS-IoT)}, pages 453--460. IEEE, 2023.

\bibitem{zhang2019theoretically}
Hongyang Zhang, Yaodong Yu, Jiantao Jiao, Eric Xing, Laurent El~Ghaoui, and Michael Jordan.
\newblock Theoretically principled trade-off between robustness and accuracy.
\newblock In {\em International conference on machine learning}, pages 7472--7482. PMLR, 2019.

\bibitem{madry2017towards}
Aleksander Madry, Aleksandar Makelov, Ludwig Schmidt, Dimitris Tsipras, and Adrian Vladu.
\newblock Towards deep learning models resistant to adversarial attacks.
\newblock {\em arXiv preprint arXiv:1706.06083}, 2017.

\bibitem{huang2022revisiting}
Shihua Huang, Zhichao Lu, Kalyanmoy Deb, and Vishnu~Naresh Boddeti.
\newblock Revisiting residual networks for adversarial robustness: An architectural perspective.
\newblock {\em arXiv preprint arXiv:2212.11005}, 2022.

\bibitem{wang2023better}
Zekai Wang, Tianyu Pang, Chao Du, Min Lin, Weiwei Liu, and Shuicheng Yan.
\newblock Better diffusion models further improve adversarial training.
\newblock In {\em International Conference on Machine Learning}, pages 36246--36263. PMLR, 2023.

\bibitem{peng2023robust}
ShengYun Peng, Weilin Xu, Cory Cornelius, Matthew Hull, Kevin Li, Rahul Duggal, Mansi Phute, Jason Martin, and Duen~Horng Chau.
\newblock Robust principles: Architectural design principles for adversarially robust cnns.
\newblock {\em arXiv preprint arXiv:2308.16258}, 2023.

\bibitem{croce2021robustbench}
Francesco Croce, Maksym Andriushchenko, Vikash Sehwag, Edoardo Debenedetti, Nicolas Flammarion, Mung Chiang, Prateek Mittal, and Matthias Hein.
\newblock {RobustBench: a standardized adversarial robustness benchmark}.
\newblock In {\em Thirty-fifth Conference on Neural Information Processing Systems Datasets and Benchmarks Track}, 2021.

\bibitem{nie2022diffusion}
Weili Nie, Brandon Guo, Yujia Huang, Chaowei Xiao, Arash Vahdat, and Anima Anandkumar.
\newblock Diffusion models for adversarial purification.
\newblock {\em arXiv preprint arXiv:2205.07460}, 2022.

\bibitem{mahmood2021robustnessvisiontransformersadversarial}
Kaleel Mahmood, Rigel Mahmood, and Marten van Dijk.
\newblock On the robustness of vision transformers to adversarial examples, 2021.

\bibitem{duda2006pattern}
Richard~O Duda, Peter~E Hart, et~al.
\newblock {\em Pattern classification}.
\newblock John Wiley \& Sons, 2006.

\bibitem{wang2019analyzingrobustnessnearestneighbors}
Yizhen Wang, Somesh Jha, and Kamalika Chaudhuri.
\newblock Analyzing the robustness of nearest neighbors to adversarial examples, 2019.

\bibitem{pang2020bag}
Tianyu Pang, Xiao Yang, Yinpeng Dong, Hang Su, and Jun Zhu.
\newblock Bag of tricks for adversarial training.
\newblock {\em arXiv preprint arXiv:2010.00467}, 2020.

\bibitem{croce2020reliable}
Francesco Croce and Matthias Hein.
\newblock Reliable evaluation of adversarial robustness with an ensemble of diverse parameter-free attacks.
\newblock In {\em International conference on machine learning}, pages 2206--2216. PMLR, 2020.

\bibitem{papernot2018cleverhans}
Nicolas Papernot, Fartash Faghri, Nicholas Carlini, Ian Goodfellow, Reuben Feinman, Alexey Kurakin, Cihang Xie, Yash Sharma, Tom Brown, Aurko Roy, Alexander Matyasko, Vahid Behzadan, Karen Hambardzumyan, Zhishuai Zhang, Yi-Lin Juang, Zhi Li, Ryan Sheatsley, Abhibhav Garg, Jonathan Uesato, Willi Gierke, Yinpeng Dong, David Berthelot, Paul Hendricks, Jonas Rauber, and Rujun Long.
\newblock Technical report on the cleverhans v2.1.0 adversarial examples library.
\newblock {\em arXiv preprint arXiv:1610.00768}, 2018.

\bibitem{rauber2017foolboxnative}
Jonas Rauber, Roland Zimmermann, Matthias Bethge, and Wieland Brendel.
\newblock Foolbox native: Fast adversarial attacks to benchmark the robustness of machine learning models in pytorch, tensorflow, and jax.
\newblock {\em Journal of Open Source Software}, 5(53):2607, 2020.

\bibitem{johnson2019billion}
Jeff Johnson, Matthijs Douze, and Herv{\'e} J{\'e}gou.
\newblock Billion-scale similarity search with {GPUs}.
\newblock {\em IEEE Transactions on Big Data}, 7(3):535--547, 2019.

\bibitem{pmlr-v202-wang23ad}
Zekai Wang, Tianyu Pang, Chao Du, Min Lin, Weiwei Liu, and Shuicheng Yan.
\newblock Better diffusion models further improve adversarial training.
\newblock In Andreas Krause, Emma Brunskill, Kyunghyun Cho, Barbara Engelhardt, Sivan Sabato, and Jonathan Scarlett, editors, {\em Proceedings of the 40th International Conference on Machine Learning}, volume 202 of {\em Proceedings of Machine Learning Research}, pages 36246--36263. PMLR, 23--29 Jul 2023.

\end{thebibliography}

\section*{Appendix}
Detailed of proof of theorem 1.
 
 Denote by $d(x+v,A)$ the distance to the correct class and
$d(x+v,B)$ the distance to the other class. We want to show that
$d(x+v,A)<d(x+v,B)$ for any $v$ such that $\|v\|<\delta/2$. Since $d(x,A)=d(x,x_0)$ we have that:
\begin{eqnarray}
d(x+v,A) &=& \min_{i \in A} d(x+v,x_i) \\
& \leq & \min_{i \in A} d(x,x_i) + \frac{\delta}{2}  \\
& =& d(x,x_0) + \frac{\delta}{2}
\end{eqnarray}

and similarly:
\begin{eqnarray}
d(x+v,B) &=& \min_{i \in B} d(x+v,x_i) \\
& \geq & \min_{i \in B} d(x,x_i) - \frac{\delta}{2}  \\
& =& d(x,x_1) - \frac{\delta}{2}
\end{eqnarray}
and since the example is $\delta$ confident
$d(x,x_1)>d(x,x_0)+ \delta$ so that $d(x+v,A)<d(x+v,B)$.

\end{document}